\documentclass{article} 
\usepackage{iclr2027_conference,times}

\usepackage{amsmath,amsfonts,bm}

\def\eqref#1{equation~\ref{#1}}

\def\1{\bm{1}}

\def\ra{{\textnormal{a}}}

\def\rx{{\textnormal{x}}}

\def\rva{{\mathbf{a}}}

\def\erva{{\textnormal{a}}}

\def\ervx{{\textnormal{x}}}

\def\rmA{{\mathbf{A}}}

\def\vmu{{\bm{\mu}}}
\def\vtheta{{\bm{\theta}}}
\def\va{{\bm{a}}}

\def\ve{{\bm{e}}}

\def\vx{{\bm{x}}}

\def\eva{{a}}

\def\mA{{\bm{A}}}

\def\mH{{\bm{H}}}
\def\mI{{\bm{I}}}
\def\mJ{{\bm{J}}}

\def\mX{{\bm{X}}}

\def\mSigma{{\bm{\Sigma}}}

\DeclareMathAlphabet{\mathsfit}{\encodingdefault}{\sfdefault}{m}{sl}
\SetMathAlphabet{\mathsfit}{bold}{\encodingdefault}{\sfdefault}{bx}{n}
\newcommand{\tens}[1]{\bm{\mathsfit{#1}}}
\def\tA{{\tens{A}}}

\def\tX{{\tens{X}}}

\def\gG{{\mathcal{G}}}

\def\sA{{\mathbb{A}}}
\def\sB{{\mathbb{B}}}

\def\sS{{\mathbb{S}}}

\def\emA{{A}}

\newcommand{\etens}[1]{\mathsfit{#1}}

\def\etA{{\etens{A}}}

\newcommand{\E}{\mathbb{E}}

\newcommand{\R}{\mathbb{R}}

\newcommand{\KL}{D_{\mathrm{KL}}}
\newcommand{\Var}{\mathrm{Var}}

\newcommand{\Cov}{\mathrm{Cov}}

\newcommand{\normltwo}{L^2}
\newcommand{\normlp}{L^p}

\newcommand{\parents}{Pa} 

\usepackage{hyperref}
\usepackage{url}

\usepackage[utf8]{inputenc} 
\usepackage[T1]{fontenc}    
\usepackage{hyperref}       
\usepackage{url}            
\usepackage{booktabs}       
\usepackage{amsfonts}       
\usepackage{nicefrac}       
\usepackage{microtype}      
\usepackage{xcolor}         

\usepackage{amsmath}
\usepackage{tabularx}
\usepackage{amsthm}
\usepackage{multirow}
\usepackage{pifont}
\usepackage{amssymb}
\usepackage{multirow}
\usepackage{graphicx}
\usepackage{caption}
\usepackage{xcolor}

\title{When Does Equivariance Help? Canonical Alignment in Neural Fluid Surrogates}

\author{Patryk Rygiel  \\
Department of Applied Mathematics \\
Technical Medical Centre \\
Cardiovascular Health Technology Centre \\
University of Twente \\
\texttt{p.t.rygiel@utwente.nl} \\
\And
Julian Suk \\
Department of Computer Science \\
Munich Center for Machine Learning \\
Technical University of Munich \\
\texttt{julian.suk@tum.de} \\
\And
Kak Khee Yeung \\
Department of Surgery \\
Amsterdam UMV, Location University of Amsterdam \\
Artherosclerosis \& Aortic diseases \\
Amsterdam Cardiovascular Sciences \\
\texttt{k.yeung@amsterdamumc.nl} \\
\And
Christoph Brune \\
Department of Applied Mathematics \\
Digital Society Institute \\
Technical Medical Centre \\
University of Twente \\
\texttt{c.brune@utwente.nl} \\
\And
Jelmer M. Wolterink  \\
Department of Applied Mathematics \\
Technical Medical Centre \\
Cardiovascular Health Technology Centre \\
University of Twente \\
\texttt{p.t.rygiel@utwente.nl} \\
}

\iclrfinalcopy 
\begin{document}
\maketitle

\begin{abstract}
Neural surrogates can accelerate computational fluid dynamics (CFD) simulations by orders of magnitude, but practical deployment in engineering and healthcare applications requires architectures that scale to high-resolution meshes and learn effectively from limited data. Explicit equivariance offers a principled inductive bias, yet its accuracy benefits may depend on the prediction task and the distribution of anatomical orientations. We investigate this dependence across three hemodynamic benchmarks with different degrees of natural canonical alignment. To support this study, we introduce the Anchored-Branched Geometric Algebra Transformer (AB-GATr), an $E(3)$-equivariant surrogate that efficiently predicts coupled surface and volume quantities. Across these benchmarks, AB-GATr consistently outperforms the evaluated non-equivariant models, including variants trained with rotational augmentation, while achieving accuracy competitive with $E(3)$-equivariant LaB-GATr at substantially lower training cost. In comparison, rotational augmentation provides inconsistent benefits across architectures and can reduce accuracy. A controlled experiment on ShapeNet-Car shows that strong canonical alignment can favor non-equivariant models, but their accuracy generally deteriorates as training orientations broaden and can decline sharply under broader test rotations. We further investigate these patterns using extended symmetry-breaking diagnostics and probes of the predictive information associated with canonical alignment across all benchmarks. Together, these results support explicit equivariance for the evaluated hemodynamic tasks with natural orientation variation, while showing that its accuracy benefits depend on the task and orientation distribution.
Our code is publicly available at~\footnote{https://github.com/PatRyg99/abgatr}.

\end{abstract}

\section{Introduction}
In‑silico modelling of physical phenomena using numerical methods has long played a central role in science and engineering. 
In computational fluid dynamics (CFD), applications in domains such as automotive design~\citep{martins_aerodynamic_2022} and medicine~\citep{Williamson2022-eo} are primarily carried out using high‑fidelity numerical solvers that approximate the Navier-Stokes (NS) equations. 
In medicine, rapid and reliable personalised assessment of hemodynamics, i.e., the way in which blood flows through arteries, can inform timely treatment decisions in patients suffering from pathologies such as stenoses and aneurysms~\citep{Taylor2013-te}.
While these methods provide accurate and physically grounded solutions, their computational cost and runtime remain bottlenecks in many real‑world settings.

In recent years, neural surrogates - deep learning-based approximators of numerical methods - have emerged as a promising alternative, offering orders‑of‑magnitude acceleration by amortising simulation cost into offline training.
Once trained, such models can estimate quantities of interest, such as pressure, velocity, or wall shear stress, in seconds.
As a result, neural surrogates have sparked increasing interest across CFD‑driven domains, including automotive aerodynamics~\citep{wu2024transolver,luo2025transolver,zhdanov2025erwin,alkin2025abupt} or hemodynamics~\citep{brehmer2023geometric,suk2024lab,suk2024operator,rygiel2023,rygiel2025wss,griffo_geometric_2026,lannelongue_physics_2026}.

Recent progress in neural surrogate modelling has been largely shaped by two research directions. 
The first focuses on scalability, seeking architectures that can handle high‑resolution, geometrically complex surface or volumetric domains without prohibitive memory or computational cost. 
This has led to more computationally efficient transformer designs~(\cite{luo2025transolver,zhdanov2025erwin}) and neural field formulations that decouple geometry encoding from pointwise prediction~\citep{Jaegle2021PerceiverIA,lu2021learning,li2023transformer,Cao2024-bn,alkin2024upt,alkin2025abupt}.
The second direction emphasises inductive bias by incorporating physical structure - such as conservation laws~\citep{raissi_physics-informed_2019,lutter2018deep,Suk2024PhysicsinformedGN,lannelongue_physics_2026} or symmetries~\citep{bronstein2021gdl} - into the model design to improve data efficiency and generalisation in domains that generally lack large training data sets. 
Among these, equivariant neural networks, which enforce symmetry constraints by construction, have shown strong empirical performance on a range of physics learning tasks~\citep{brandstetter2022clifford,brandstetter2022segnn,brehmer2023geometric,suk2024mesh,suk2024lab,catalani2025geometry,kim2026eqgino,wessels2025grounding}.

Across scientific machine learning applications, the necessity of including inductive biases in architecture design is increasingly under scrutiny, particularly as training dataset sizes increase in scale~\citep{brehmer2024doesequivariancematterscale}.
Recent notable works in molecular design suggest that, with sufficient data, neural networks can implicitly learn physical symmetries, making explicit equivariance unnecessary~\citep{abramson_accurate_2024,wang2024swallowing}.
However, in practice, this assumption does not hold uniformly across application domains and can affect model behaviour~\citep{pmlr-v221-moskalev23a}.
In medical and engineering CFD problems in particular, two challenges compound. 
First, training samples are scarce because acquisition and simulation are costly~\citep {alkin2025abupt,rygiel2025activelearning}. 
Second, training data can vary widely, and the distribution of geometric orientations may influence the usefulness of equivariance as an inductive bias.

\begin{figure}[t!]
    \centering
    \includegraphics[width=\textwidth]{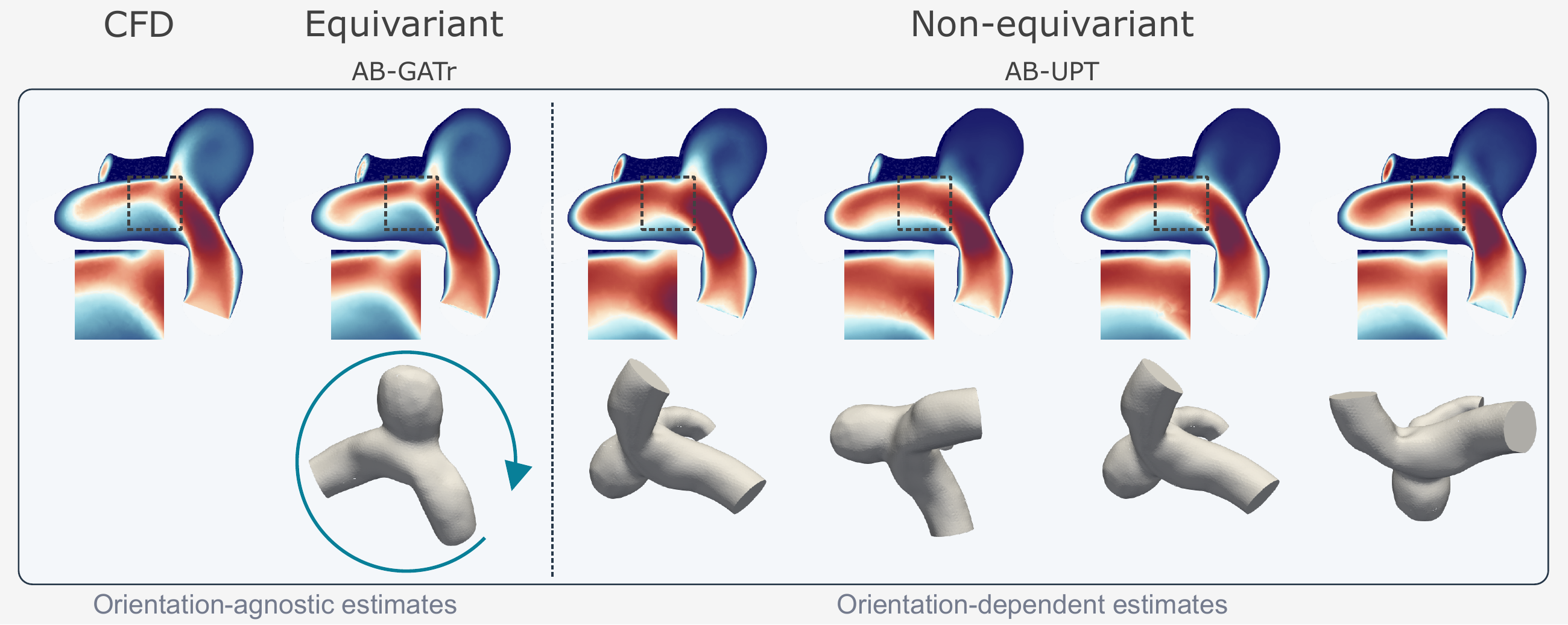}
    \caption{Vascular systems can realistically exhibit different orientation variability. Non-equivariant models treat all these orientations as different samples; hence, the estimates vary between them and may lead to degraded performance. On the other hand, equivariant models are agnostic to these orientation changes, yielding consistent and more accurate estimates.}
    \label{fig:main}
\end{figure}

We argue that the \textit{canonical alignment} of a dataset, i.e., the extent to which geometries share similar orientations, determines the need for equivariance. 
For example, automotive aerodynamics benchmarks~\citep{shapenetcar,ashton2025drivaerml} typically exhibit strong canonical alignment through standardised wind-tunnel configurations. 
In such datasets, enforcing equivariance might be counterproductive as models lose the ability to exploit canonical alignment-induced cues and therefore attain suboptimal in-distribution performance~\citep{vadgama2025probingequivariancesymmetrybreaking,lawrence2026to}.
In contrast, hemodynamic datasets naturally exhibit different degrees of alignment: Aortas retain a preferred anatomical orientation, whereas intracranial artery geometries exhibit substantially greater orientation variability. 
In such problems, relying solely on data to recover fundamental physical symmetries may be both inefficient and unreliable~(see Fig.~\ref{fig:main}).

This work makes two contributions. 
First, we systematically evaluate the relation between canonical alignment and equivariance across three hemodynamic benchmarks, and assess robustness to orientation-distribution shifts.
Building on~\cite{lawrence2026to}, we extend symmetry diagnostics with rotation-dependent response curves and integrated scores, and apply task-dependent probes to examine how canonical alignment relates to surrogate performance.
Second, we introduce \textbf{Anchored-Branched Geometric Algebra Transformer (AB-GATr)}, an $E(3)$-equivariant neural surrogate that combines scalable anchor-attention and multi-branching~\citep{alkin2025abupt} with explicit $E(3)$-equivariance achieved through projective geometric algebra~\citep{brehmer2023geometric}. 
AB-GATr outperforms non-equivariant models on hemodynamic tasks and achieves better or on-par performance relative to $E(3)$-equivariant LaB-GATr~\citep{suk2024lab} at substantially lower training cost.

\section{Materials \& Methods}
We first describe the supervised CFD learning problem and distinguish solution-map equivariance from symmetry in the data distribution. 
We then introduce AB-GATr and the benchmarks used to investigate canonical alignment and the benefits of explicit equivariance under orientation changes.

\subsection{Computational Fluid Dynamics (CFD)}
CFD requires solving the incompressible Navier-Stokes (NS) equations, partial differential equations (PDEs) that govern fluid dynamics:

\begin{align*}
    \rho \left( \frac{\partial u}{\partial t} + (u \cdot \nabla) u \right) = - \nabla p + \mu \nabla^2 u + g, \: \: \: \: \nabla \cdot u = 0
\end{align*}

Here, $u$ denotes the velocity field, $p$ the pressure field, $g$ the external forces, and $\mu$ and $\rho$ the fluid viscosity and density, respectively. 
Numerical solvers approximate the solution on a discretised domain $\Omega \subset \mathbb{R}^3$ subject to prescribed boundary conditions $\mathcal{B}$ and initial conditions.
The resulting fields (velocity, pressure, etc.) are represented on the computational domain $\Omega$ and provide the targets for surrogate learning.

\subsection{Equivariant neural surrogates for CFD}
\label{sec:equivariance}
Neural surrogates provide a data‑driven alternative to classical numerical solvers by learning to approximate the mapping from the domain and boundary conditions to the PDE solution. 
We view a neural surrogate as approximating the mapping $f_{\theta} \in \mathcal{F}: (\mathcal{X}, \mathcal{B}) \mapsto \mathcal{Y}$, where $\mathcal{X}$ denotes the input signal (e.g., Cartesian coordinates or normal vectors) defined over the spatial domain $\Omega$, $\mathcal{B}$ represents boundary conditions, and $\mathcal{Y}$ the target solution.

Many PDEs exhibit intrinsic structure that neural surrogates can exploit through appropriate inductive biases imposed on the hypothesis class $\mathcal{F}$. 
These biases restrict $\mathcal{F}$ to functions consistent with known physical priors, thereby improving generalisation and data efficiency~\citep{bronstein2021gdl,brehmer2024doesequivariancematterscale}. 
A prominent example is equivariance and invariance to symmetry, which is particularly relevant in CFD: assuming the absence of external forces, or that such forces respect the same symmetries as the domain, the incompressible 3D NS operator is equivariant to the Euclidean group $E(3)$.
Concretely, for any group action $\rho \in E(3)$, i.e., rotating, translating, or reflecting the spatial domain changes the solution in a well-defined way: $u(\rho x, t) = \rho u(x, t), \: p(\rho x, t) = p(x, t)$.
Where the velocity field transforms equivariantly as a vector field, while the pressure transforms invariantly as a scalar field. 
Incorporating equivariance into neural surrogate models enables them to respect the underlying physics by construction.

However, an operator's equivariance does not automatically imply equivariance of the solution map. 
For a neural surrogate $f$ to be equivariant, the group action must jointly transform the input signal $\mathcal{X}$ and the boundary conditions $\mathcal{B}$:
$f\big(\rho\mathcal{X}, \rho\mathcal{B}\big)
= \rho f\big(\mathcal{X}, \mathcal{B}\big)$.
If either the input signal or the boundary conditions violate the assumed symmetry, equivariance is broken, and the NS solution no longer transforms consistently under $E(3)$.
This effect is amplified if we never explicitly supply boundary conditions to the model but require it to learn them as a function of the input signal $ f(\mathcal{X}, \mathcal{B}(\mathcal{X}))$.


This distinction highlights that symmetry is a property of the entire problem formulation, not of the governing equations alone. 
Consequently, the effectiveness of equivariant inductive biases depends strongly on both the problem and the data's rate of canonical alignment, echoing recent findings by \citet{lawrence2026to}.

\subsection{Anchored-Branched Geometric Algebra Transformer (AB-GATr)}

Designing neural surrogates for real‑world CFD involves simultaneously addressing two competing challenges: scalability to large, high‑resolution surface and volume meshes, and a physically meaningful inductive bias.
This work bridges these paradigms by proposing a scalable, equivariant architecture.

\begin{figure}[t!]
    \centering
    \includegraphics[width=\textwidth]{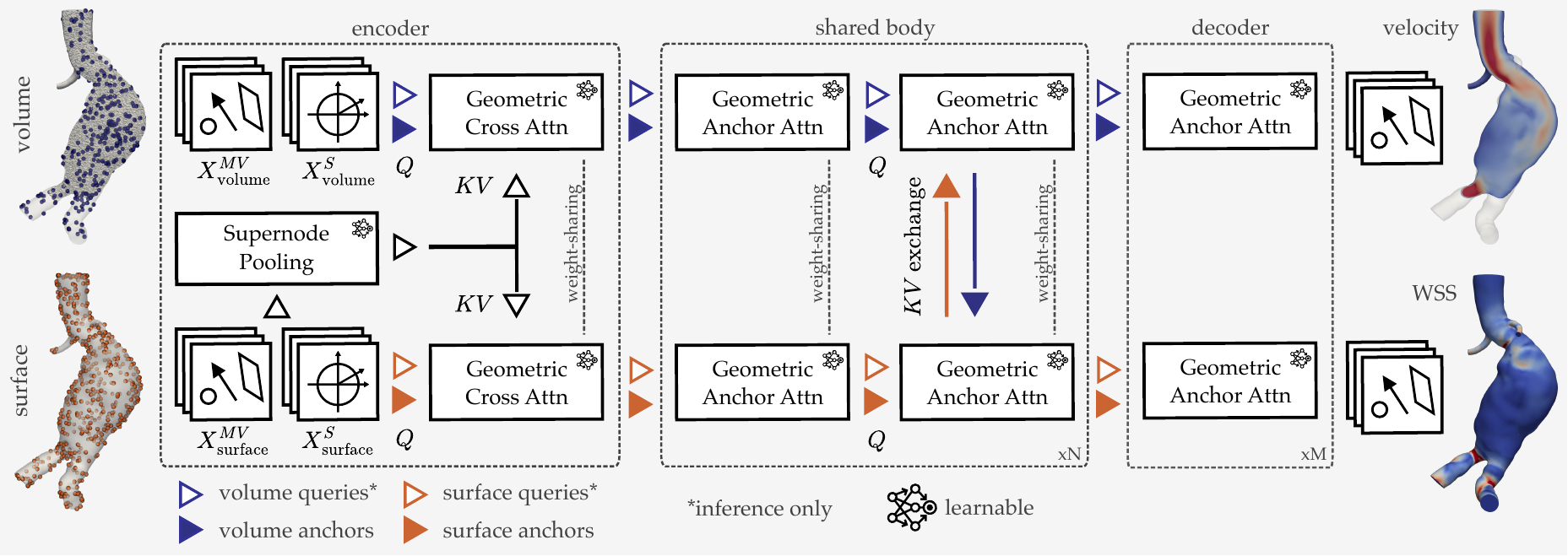}
    \caption{AB‑GATr is a transformer-based $E(3)$-equivariant neural field that jointly models volumetric and surface quantities. Volume and surface nodes are embedded using geometric algebra multivectors $X^{MV}$ alongside sine-cosine encoded scalar features $X^{S}$. 
    Sparse surface and volume anchors are uniformly sampled and serve as keys and values for geometric anchor attention throughout the network. At inference, additional query tokens are concatenated to the query set, while keys and values remain restricted to anchors. In the shared body, branch-specific geometric anchor attention is interleaved with a KV-exchange variant, in which key and value tokens are taken from opposite branches. Branch-specific decoders are applied at the output stage, where target quantities are extracted from the resulting multivectors.}
    \label{fig:abgatr}
\end{figure}
We introduce the \textbf{Anchored‑Branched Geometric Algebra Transformer (AB‑GATr)}~(Fig.~\ref{fig:abgatr}), an $E(3)$-equivariant neural surrogate built upon Anchored‑Branched Universal Physics Transformer (AB‑UPT;~\cite{alkin2025abupt}). 
AB‑GATr retains AB-UPT's anchor attention and multi-branching for processing large CFD meshes, while incorporating equivariance through projective geometric algebra (PGA) as introduced through GATr architecture~(\cite{brehmer2023geometric}).
The following sections describe the core components of the AB‑GATr architecture.

\subsubsubsection{\textbf{PGA \& positional encoding}}: Like AB‑UPT, AB‑GATr operates as a transformer‑based neural field that processes mesh nodes as tokens. 
In AB‑UPT, geometric information, such as Cartesian coordinates or surface normals, is encoded in an unstructured manner by concatenating node‑level features into a single vector. 
This approach does not preserve geometric structure and is inherently non‑equivariant.
In contrast, AB‑GATr adopts PGA to provide a structured geometric representation, associating each token with a $16$‑dimensional multivector $X^{MV}$ capable of encoding geometric primitives such as points or directions~(see App~\ref{app:features}).
Operations on multivectors respect the actions of the Euclidean group $E(3)$, ensuring equivariant processing of geometric information.
Similar to~\cite{brehmer2023geometric}, AB‑GATr additionally supports auxiliary scalar features $X^{S}$ for encoding non‑geometric attributes.

Transformer architectures typically rely on positional encodings to inject spatial context and enrich the input representation. 
AB‑UPT employs sine–cosine embeddings~\citep{vaswani2023} and rotary positional encoding (RoPE)~\citep{su_roformer_2024}, which cannot be directly applied to geometric algebra representations without breaking equivariance. 
To preserve $E(3)$-equivariance in AB‑GATr, positional encodings are therefore restricted only to the auxiliary scalars.

\subsubsubsection{\textbf{Geometric anchor-attention}}: Self‑attention is a central component of transformer architectures, modelling all pairwise interactions among the input tokens $X$ of size $N$.
This results in $O(N^2)$ computational complexity, which becomes prohibitively expensive for high-resolution computational meshes in real-world scenarios, making full self‑attention impractical.

To address this scale challenge, \cite{alkin2025abupt} introduced the anchor-attention mechanism. 
Anchor-attention serves two primary purposes. 
Firstly, it reduces the complexity to $O(MN)$ by uniformly sampling a small subset of tokens $A \subset X$, with $|A| = M \ll N$, designated as \textit{anchors}. 
These anchors act as the keys and values in the attention mechanisms, while all tokens act as queries. 
Secondly, because non‑anchor tokens attend only to anchors, each non‑anchor token can be processed independently, satisfying the requirements of zero-shot super-resolution (a neural field property).

We extend the idea of anchor-attention to AB‑GATr by introducing \textit{geometric anchor-attention}, an $E(3)$-equivariant analogue.
Here, input tokens $X = (X^{MV}, X^{S})$, are represented by a tuple of multivectors $X^{MV}$ and corresponding auxiliary scalars $X^S$.
Attention weights of geometric attention are a linear combination of the following three elements: the invariant inner product of the geometric algebra $\langle \cdot, \cdot \rangle$, the distance-aware inner product of nonlinearities $\phi(\cdot), \psi(\cdot)$, and the standard Euclidean inner product of the auxiliary scalars with RoPE encoding $\text{R}$:

\begin{align*}
    \text{W}(k, q) := \text{Softmax} \left( \frac{\alpha \sum_{c}{\langle q^{MV}_{c}, k^{MV}_{c} \rangle} + \beta \sum_c{\phi (q^{MV}_{c}) \cdot \psi(q^{MV}_{c})} + \gamma \sum_c{\text{R}(q^{S}_{c})\text{R}(k^{S}_{c})}}{\sqrt{13n_{MV} + n_{S}}} \right)
\end{align*}

Where $n_{MV}, n_{S}$ are respectively the dimensionalities of multivectors and scalars, $c$ label channels, and $\alpha, \beta, \gamma > 0$ head-specific learnable weights (for more details, we refer to~\cite{brehmer2023geometric}).

Attention weights $W(\cdot, \cdot)$ can be utilised for both geometric self-attention and geometric anchor-attention in the following manner by employing equilinear projection layers $K(\cdot)$, $Q(\cdot)$, and $V(\cdot)$:

\begin{align*}
    \text{GeometricSelfAttention}(X) &:= W(K(X), Q(X)) V(X) \\
    \text{GeometricAnchorAttention}(X, A) &:= W(K(A), Q(X)) V(A) \\ \\
    A = (X^{MV}[\mathcal{I}_A], X^S[\mathcal{I}_A]) \:\:\: \mathcal{I}_A \thicksim \: &\mathrm{Uniform}(\{1,2,...,N\}, M)
\end{align*}

\subsubsubsection{\textbf{Supernode pooling}}: To inject manifold boundary into the token representations, both volume and surface tokens attend to a set of surface‑pooled supernodes~(see Fig.~\ref{fig:abgatr})
The supernodes $S=A_{\text{surface}}$ correspond to the same surface indices sampled as anchors.
For each supernode  $a \in A_{\text{surface}}$ a local neighborhood is constructed via a radius‑ball: $\mathcal{N}_{r}(a) := \{b \in X_{\text{surface}} \: | \:\|a-b\|_2 \leq r\}$, over which message‑passing pooling is performed.
During pooling, message signals are augmented with the relative translation vector, encoded as a \textit{translation} primitive~(App~\ref{app:features}), and its magnitude as an auxiliary scalar, enabling boundary‑aware aggregation while preserving $E(3)$-equivariance.

\subsubsubsection{\textbf{Multi-branching}}:
Following~\cite{alkin2025abupt}, we introduce a multi‑branch design in which separate branches process volume and surface tokens while sharing a subset of weights and exchanging information through attention keys and values (KV exchange).
AB‑GATr adopts this design and extends it to the equivariant setting by jointly exchanging both multivectors $X^{MV}$ and auxiliary scalars $X^{S}$ across branches.
As illustrated in Fig.~\ref{fig:abgatr}, the \textit{encoder} processes volume and surface tokens with shared weights, followed by a \textit{shared body} in which standard geometric anchor-attention is interleaved with a KV exchange to enable cross‑branch interaction.
In the \textit{decoder} stage, the branches become fully independent, allowing for regression of volumetric and surface quantities.

\subsection{CFD benchmarks}
\label{sec:benchmarks}

\begin{figure*}[t!]
    \centering
    \includegraphics[width=\textwidth]{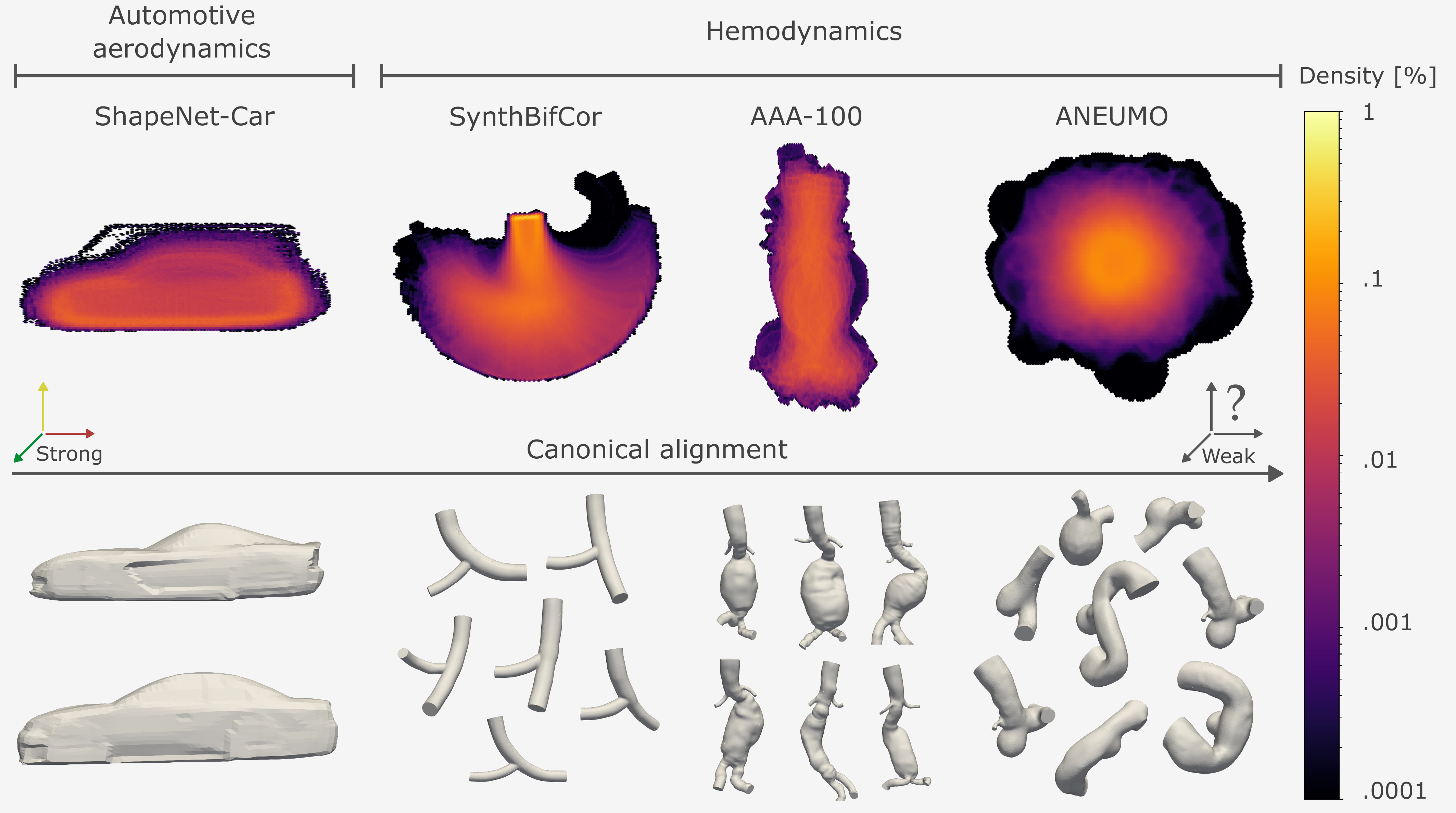}
    \caption{Canonical alignment across different benchmarks. The first row represents a histogram of point occupancy across the whole dataset. The second row shows representative geometries sampled from the dataset kept in the same orientation as the histogram.}
    \label{fig:data}
\end{figure*}

As discussed in Section~\ref{sec:equivariance}, equivariance of the solution map arises only when input signal and boundary conditions jointly preserve symmetry.
Consequently, equivariant inductive biases may yield different results across benchmarks with varying degrees of natural canonical alignment.
This work uses the following hemodynamic benchmarks: 

\textbf{SynthBifCor}~\citep{suk2024mesh} contains $2,000$ synthetic coronary bifurcation geometries derived from anatomical atlas measurements. 
Steady simulations use an inflow velocity of $11.8$ cm/s at Reynolds number $100$. 
The geometries share an inlet orientation but vary in the orientation of branch configuration.


\textbf{AAA-100}~(\cite{rygiel2024aaa100,rygiel2025wss})contains $100$ patient-specific abdominal aortic aneurysm (AAA) geometries reconstructed from computed tomography scans.
Transient CFD simulations use a prescribed inflow waveform with a peak flow rate of $80$ ml/s~\footnote{\label{fn:bc} AAA-100 and ANEUMO provide simulations under different boundary conditions. Since we are interested in the influence of geometric variability only, we use only a single set of fixed boundary conditions for both.}.
The geometries retain a preferred anatomical orientation with variation in tilt and axial rotation.

\textbf{ANEUMO}~\citep{li2025aneumo} contains $10,660$ synthetic intracranial aneurysm geometries derived from $400$ patient‑specific anatomies~(\cite{Juchler2022-ro}) and simulated under steady inflow conditions with mass flow of $0.0010$ kg/s~\footref{fn:bc}.
The geometries exhibit broad orientation variability and an apparent lack of canonical alignment.

These hemodynamic benchmarks share the property that they represent \textit{internal flows}, which are predominantly intrinsic to geometry, with absolute orientation carrying little or no physical meaning.
In contrast, \textit{external flows} as seen in automotive aerodynamics are defined within fixed global reference frames, which usually results in canonical alignment and explicit symmetry breaking imposed by the wind-tunnel setup. To assess the differences between these two settings, we include a fourth benchmark: 

\textbf{ShapeNet‑Car}~\citep{shapenetcar} contains 3D car geometries from ShapeNet~\citep{chang2015shapenet} with steady incompressible CFD simulations at an inflow speed of $72$ km/h. 
All car geometries share a common canonical orientation, with a fixed inflow direction reflecting a wind‑tunnel experimental setup.

These benchmarks span distinct patterns of canonical alignment and geometric variability that naturally arise in real-world CFD problems, providing a robust testbed for assessing the influence of equivariance in neural surrogate design.
More details about the datasets are provided in App.~\ref{app:datasets}. 

To ground our observations, we quantify canonical alignment by extending symmetry-breaking diagnostics of~\citet{lawrence2026to} to examine how distinguishable the data remain from randomly rotated counterparts at original orientations and when their variability increases.
Specifically, we consider an integral of symmetry response across orientation variability as a measure of data canonical alignment~(App.~\ref{app:distirbutional-symmetry}).

\section{Experiments \& Results}
We compare AB-GATr with its non-equivariant counterpart AB-UPT, the non-equivariant Transolver++~\citep{luo2025transolver}, and the $E(3)$-equivariant LaB-GATr~\citep{suk2024lab}, which utilises the same equivariance formulation as AB-GATr.
To thoroughly benchmark the notion of equivariance, we fix AB-UPT and AB-GATr to follow the same architecture size, i.e., number of attention blocks, number of anchors/tokens, etc., with the only difference being the size of the hidden dimension and number of attention heads due to the difference between linear and geometric algebra layers~(training configurations and resource requirements are reported in App.~\ref{app:training}).

We evaluate each model on the original test geometries and on rotated versions of the same cases. 
We train both non-equivariant architectures with and without $SO(3)$ rotational augmentation; augmented variants are denoted by the superscript “rot.” 
The original test sets retain each benchmark’s native orientation distribution, whereas the additional rotation evaluation tests sensitivity to a changed orientation distribution.


\begin{table}
  \caption[Short Heading]{We report \texttt{median} of mean relative L2 error [\%] of volume velocity ($u$) and surface WSS ($\tau$) over $3$ runs for SynthBifCor and ANEUMO, and over $10$-fold cross-validation for AAA-100. Best-performing model in bold, with an underline indicating whether its performance gain is statistically significant compared to the second-best ($p < 0.05$ in the Wilcoxon test).}
  \label{hemodynamics-table}
  \centering
  \begin{tabular}{lllllllllll}
    \hline
        \toprule 
        & \multicolumn{4}{c}{\textbf{SynthBifCor}} & \multicolumn{4}{c}{\textbf{AAA-100}} & \multicolumn{2}{c}{\textbf{ANEUMO}} \\
        
        & \multicolumn{2}{c}{\textit{original}} & \multicolumn{2}{c}{\textit{rotated}} & \multicolumn{2}{c}{\textit{original}} & \multicolumn{2}{c}{\textit{rotated}} & \multicolumn{1}{c}{\textit{original}} & \multicolumn{1}{c}{\textit{rotated}} \\

        \cmidrule(r){2-3} 
        \cmidrule(r){4-5}
        \cmidrule(r){6-7}
        \cmidrule(r){8-9}
        \cmidrule(r){10-10}
        \cmidrule(r){11-11}
        
        & \multicolumn{1}{c}{$\boldsymbol{u}$} & \multicolumn{1}{c}{$\boldsymbol{\tau}$} & \multicolumn{1}{c}{$\boldsymbol{u}$} & \multicolumn{1}{c}{$\boldsymbol{\tau}$} & 
        \multicolumn{1}{c}{$\boldsymbol{u}$} & \multicolumn{1}{c}{$\boldsymbol{\tau}$} &  \multicolumn{1}{c}{$\boldsymbol{u}$} & 
        \multicolumn{1}{c}{$\boldsymbol{\tau}$} & \multicolumn{1}{c}{$\boldsymbol{u}$} & \multicolumn{1}{c}{$\boldsymbol{u}$} \\

        \midrule

        Transolver++$^{\dagger}$ & \multicolumn{1}{c}{1.25} & \multicolumn{1}{c}{8.62} & \multicolumn{1}{c}{103} & \multicolumn{1}{c}{104} & \multicolumn{1}{c}{71.4} & \multicolumn{1}{c}{81.8} & \multicolumn{1}{c}{123} & \multicolumn{1}{c}{110} & \multicolumn{1}{c}{48.1} & \multicolumn{1}{c}{48.0}\\
        
        AB-UPT & \multicolumn{1}{c}{2.02} & \multicolumn{1}{c}{8.27} & \multicolumn{1}{c}{96.7} & \multicolumn{1}{c}{92.0} & \multicolumn{1}{c}{57.5} & \multicolumn{1}{c}{71.5} & \multicolumn{1}{c}{109} & \multicolumn{1}{c}{108} & \multicolumn{1}{c}{25.5} & \multicolumn{1}{c}{26.1}\\

        \midrule

        Transolver++$^{\dagger, \text{rot}}$ & \multicolumn{1}{c}{2.02} & \multicolumn{1}{c}{8.96} & \multicolumn{1}{c}{2.01} & \multicolumn{1}{c}{8.98} & \multicolumn{1}{c}{92.1} & \multicolumn{1}{c}{66.6} & \multicolumn{1}{c}{88.8} & \multicolumn{1}{c}{65.4} & \multicolumn{1}{c}{20.0} & \multicolumn{1}{c}{19.9}\\
        
        AB-UPT$^{\: \text{rot}}$ & \multicolumn{1}{c}{1.37} & \multicolumn{1}{c}{7.46} & \multicolumn{1}{c}{1.35} & \multicolumn{1}{c}{7.47} & \multicolumn{1}{c}{46.1} & \multicolumn{1}{c}{44.6} & \multicolumn{1}{c}{45.3} & \multicolumn{1}{c}{44.6} & \multicolumn{1}{c}{13.0} & \multicolumn{1}{c}{13.0} \\

        \midrule

        LaB-GATr$^\dagger$ & \multicolumn{1}{c}{1.21} & \multicolumn{1}{c}{7.46} & \multicolumn{1}{c}{1.21} & \multicolumn{1}{c}{7.46} & \multicolumn{1}{c}{42.0} & \multicolumn{1}{c}{47.3} & \multicolumn{1}{c}{42.0} & \multicolumn{1}{c}{47.3} & \multicolumn{1}{c}{\underline{\textbf{10.4}}} & \multicolumn{1}{c}{\underline{\textbf{10.4}}} \\
        
        AB-GATr & \multicolumn{1}{c}{\underline{\textbf{0.98}}} & \multicolumn{1}{c}{\underline{\textbf{6.87}}} & \multicolumn{1}{c}{\underline{\textbf{0.98}}} & \multicolumn{1}{c}{\underline{\textbf{6.87}}} & \multicolumn{1}{c}{\underline{\textbf{39.0}}} & \multicolumn{1}{c}{\underline{\textbf{38.4}}} & \multicolumn{1}{c}{\underline{\textbf{39.0}}} & \multicolumn{1}{c}{\underline{\textbf{38.4}}} & \multicolumn{1}{c}{10.9} & \multicolumn{1}{c}{10.9}  \\

        \bottomrule
        \multicolumn{11}{l}{\small $\dagger$ - trained separately for surface and volume. $\quad \quad$ $\mathrm{rot}$ - model trained with SO(3) data augmentation.}
  \end{tabular}
\end{table}

\subsubsubsection{\textbf{Hemodynamic benchmarks}}: 
Table~\ref{hemodynamics-table} reports relative L2 errors for wall shear stress ($\tau$) and velocity ($u$) estimation across hemodynamic datasets.
For benchmarks with a notion of canonical alignment, namely SynthBifCor and AAA‑100, non-equivariant models perform substantially worse on \textit{rotated} datasets when trained without rotational augmentation.
In contrast, for the ANEUMO benchmark, which lacks any canonical alignment, performance on oriented and rotated geometries remains similar.

Rotation augmentation can reduce the sensitivity of a model to orientation changes, but its effect on performance depends on the architecture and dataset.
First, as expected, we observe that augmentation allows non-equivariant models to match performance between \textit{original} and \textit{rotated} test sets.
However, performance gains from data augmentation are not consistent across models and benchmarks.
AB-UPT's performance consistently improves across all hemodynamic benchmarks, even with original alignment.
Meanwhile, Transolver++'s performance on original alignment degrades for SynthBifCor and velocity prediction in AAA-100.
Nevertheless, even with augmentation, neither model can close the performance gap to AB‑GATr on any benchmark and task.

LaB-GATr results suggest that the gains from explicit equivariance extend beyond specific architectural choices.
Despite the architectural differences between LaB-GATr and AB-GATr, both equivariant models mostly outperform the non-equivariant baselines.
Moreover, the equivariant models are the best-performing ones on every benchmark: AB-GATr leads on SynthBifCor and AAA-100, while LaB-GATr achieves a slightly lower velocity error on ANEUMO ($10.4\%$ versus $10.9\%$). 
These results identify explicit equivariance as a shared feature of the best-performing models and support symmetry preservation as an important design choice across the evaluated architectures.

\begin{figure}[t!]
    \centering
    \includegraphics[width=\textwidth]{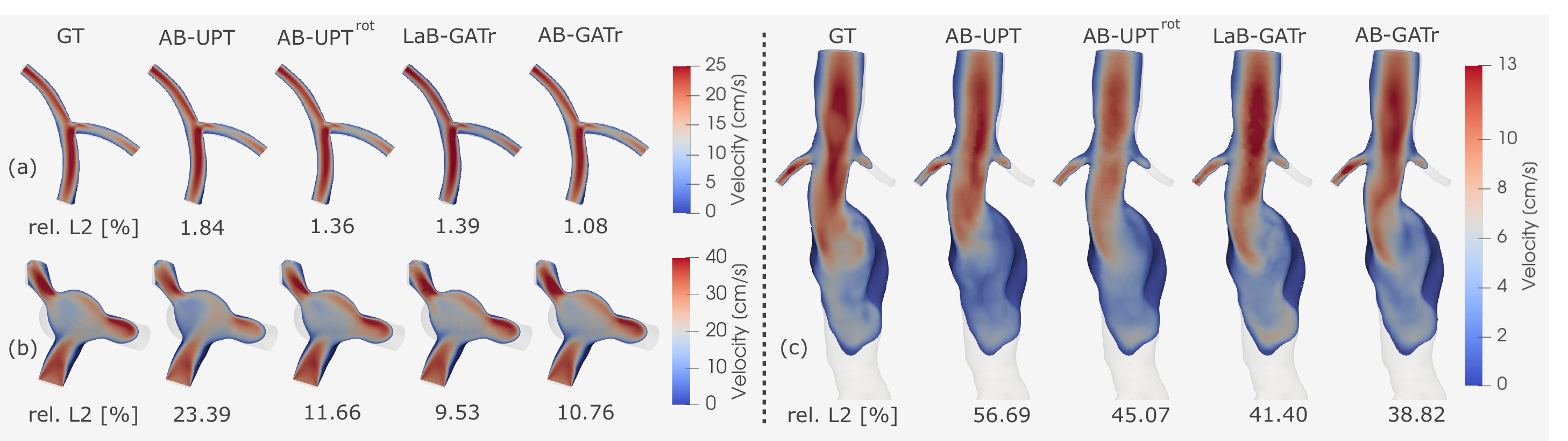}
    \caption{Qualitative comparison of neural surrogate velocity predictions across hemodynamic benchmarks: (a) SynthBifCor, (b) ANEUMO, and (c) AAA-100. In case (c), we visualise the time-averaged velocity field, since the AAA-100 benchmark consists of transient simulations.}
    \label{fig:visuals}
\end{figure}

Fig.~\ref{fig:visuals} illustrates velocity field predictions for the hemodynamic benchmarks across the anchored-branched model family. 
Predictions on SynthBifCor (a) are largely visually indistinguishable across models, consistent with the low mean relative L2 errors (< $2\%$).
Differences are more apparent on ANEUMO (b) and AAA-100 (c), which consist of more realistic geometries and hence more complex flow patterns.
Compared with LaB-GATr and AB-GATr, AB-UPT shows the weakest performance in both variants. 
For ANEUMO, errors become more pronounced further from the inlet, where velocity magnitudes are underestimated, and complex flow patterns are insufficiently captured. 
A similar behaviour is observed on AAA-100, where side branches and the aneurysm sac lack finer and high-frequency flow details present in the CFD solution.
In particular, AB-GATr yields the smoothest and most coherent predictions, with the greatest improvements in the leftmost branch in example (b) and throughout the entire geometry in example (c). 
Additional predictions and residual maps, including Transolver++, are provided in App.~\ref{app:extended-visuals}.

\subsubsubsection{\textbf{Controlled variation of canonical alignment:}}
Because hemodynamic datasets naturally have imperfect canonical alignment, we cannot fully control its rate within a single benchmark.
Hence, to thoroughly examine alignment within a fixed benchmark, we construct a controlled experiment on ShapeNet-Car.
Starting from its perfect canonical orientation reflecting a wind-tunnel setup, we systematically vary the range of rotational augmentation during training and evaluation.
Specifically, we train and evaluate AB‑UPT using rotations sampled up to $[5^{\circ},15^{\circ},30^{\circ},60^{\circ},90^{\circ},180^{\circ}]$. 
Figure~\ref{fig:heatmap} summarises the results and reveals three trends.
First, AB-UPT outperforms AB-GATr on the original canonically aligned data and under restricted rotation ranges.
Second, broadening the training orientation range generally increases AB-UPT’s errors across all orientations represented during training.
Third, evaluation beyond that range produces severe failures of generalisation, with errors increasing sharply. 
AB-GATr maintains unchanged errors across the tested rotations.
Together, these results indicate that problems with strong canonical alignment may benefit from a lack of equivariance, and enforcing or learning symmetry preservation can deteriorate in‑distribution performance while ensuring reliable generalisation when evaluated under unseen orientations.

\begin{figure}[t!]
    \begin{minipage}{0.47\textwidth}
\centering
    \includegraphics[width=\linewidth]{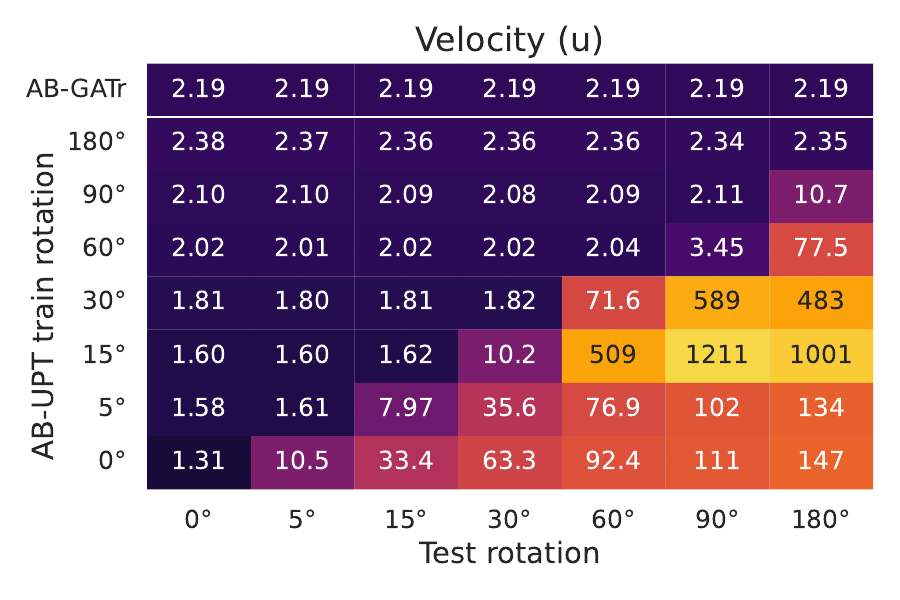}
    \caption{Median relative L2 error [\%] across $5$ runs on ShapeNet-Car for different training and testing rotation rates. Rotation rate $x^{\circ}$ indicates that random rotations up to $x^{\circ}$ were used.}
    \label{fig:heatmap}
    \end{minipage}
    \hfill
    \begin{minipage}{0.51\textwidth}
        \centering
        \includegraphics[width=\textwidth]{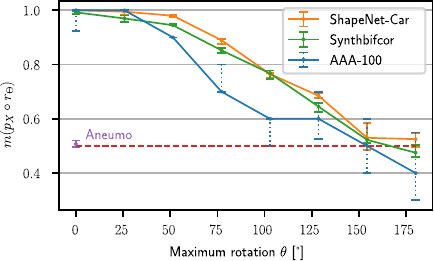}
        \caption{Symmetry response; $m(p_X \circ r_\Theta)$ measures alignment ($\uparrow$). We show lower and upper quartiles over four runs (Aneumo, ShapeNet-Car, SynthBifcor) or 10-fold cross-validation (AAA-100).}
        \label{fig:distributional-symmetry-breaking}
    \end{minipage}
\end{figure}

\subsubsubsection{\textbf{Diagnosing distributional symmetry breaking:}}
\label{sec:symmetry-breaking}
As observed with ShapeNet-Car, the dataset's natural orientation can break the symmetry of the assigned problem and benefit non-equivariant models.
\citet{lawrence2026to} proposed distributional symmetry breaking to diagnose this mechanism.
We extend their formulation of binary classification accuracy $m(p_X)$ under random group transformation of the samples (with probability $\frac{1}{2}$) by additionally rotating all samples up to a maximum angle $\theta$, which we denote by $m(p_X \circ r_\Theta)$.
Here, $\Theta \sim \mathcal{U}_{[0, \theta]}$ and $r_\Theta \in \{R_{\theta'} \in \mathrm{SO(3)} \mid \theta' \le \Theta\}$ where rotations $R_{\theta'}$ are described by their angle $\theta'$.
Figure~\ref{fig:distributional-symmetry-breaking} shows the symmetry response curve for $\theta \in [0, \pi]$.
Aneumo is the only dataset without distributional symmetry breaking $m(p_X) \approx 0.5$ (i.e., random guessing), while ShapeNet-Car, SynthBifCor and AAA-100 are perfectly distinguishable $m(p_X \circ r_\theta) \approx 1$ for $\theta = 0$.
Nevertheless, SynthBifCor and AAA-100 come with a weaker notion of canonical alignment; their symmetry response drops at a higher rate with increasing $\theta$.
We can quantify this phenomenon by integrating over the symmetry response curve, which provides an a priori measure of canonical alignment (App.~\ref{app:distirbutional-symmetry}).
Following the flow chart proposed by \citet{lawrence2026to}, we also probe task-useful distribution symmetry breaking, which we find to correlate with our empirical observations on AB-GATr compared to AB-UPT (App.~\ref{app:distirbutional-symmetry}).

\section{Discussion}
Dataset scale has often been cited as a key factor when assessing the benefits of equivariance~\citep{brehmer2024doesequivariancematterscale}. 
While scale matters, our results suggest it is not sufficient on its own to determine whether equivariant or non‑equivariant models are preferable. 
In particular, the ANEUMO benchmark is an unusually large dataset ($\sim$10k simulations), yet explicit equivariance remains beneficial.
Instead, we find that the degree of canonical alignment induced by the problem formulation plays a comparably important role.
Across the hemodynamic benchmarks, characterised by different orientation variation, equivariant models consistently outperform non‑equivariant ones - even with rotation augmentation.
Conversely, in strongly aligned settings such as ShapeNet‑Car, non‑equivariant models can exploit alignment‑specific regularities to improve in‑distribution performance, an effect also noted by~\cite{lawrence2026to,vadgama2025probingequivariancesymmetrybreaking}. 
The stronger symmetry response on ShapeNet-Car is consistent with AB-UPT’s advantage under canonical alignment, while the results on SynthBifCor and AAA-100 demonstrate that detectable alignment does not preclude benefits from equivariance.



\subsubsubsection{\textbf{Limitations and future work}}:
Our conclusions are limited to the evaluated CFD tasks, prescribed boundary-condition regimes, and model configurations.
We evaluated a single notion of equivariance and architectural design, which we both chose for their computational efficiency and performance.
Alternative equivariant methods such as GEM-CNN~\citep{haan2021gemcnn}, SEGNN~\citep{brandstetter2022segnn} or $SE(3)$-Transformer~\citep{fuchs2020se3transformers} have been previously reported to fare worse in both performance and computational efficiency~\citep{brehmer2023geometric,suk2024lab}.
While EqGINO~\citep{kim2026eqgino} and Platonic Transformers~\citep{islam2026platonic} restrict exact rotational equivariance to finite subgroups, PONITA~\citep{bekkers2024fast} introduces discretisation-dependent equivariance errors.
Consistent improvements in hemodynamic prediction accuracy due to explicit equivariance come at the cost of additional computational overhead. 
However, we found that AB-GATr substantially reduced computation cost compared to LaB-GATr~(App.~\ref{app:resources}), narrowing the gap with the non-equivariant AB-UPT model. 
While future work should seek to narrow this gap further, the $2-11$-fold increase in training time between AB-UPT and AB-GATr would be acceptable in most practical applications. 
On the other hand, other scalable architectural designs like Transolver++~\citep{luo2025transolver} or Erwin~\citep{zhdanov2025erwin} have significantly higher computational time and memory burden at our benchmark scale, which limits adoption of equivariance in these frameworks (App.~\ref{app:resources}).
Finally, our findings are empirical and do not yet provide a principled criterion for determining when equivariance is beneficial - a key step toward guiding model design in real-world applications.


\subsubsubsection{\textbf{Conclusions}}:
We introduced AB-GATr, a scalable $E(3)$-equivariant CFD surrogate, and showed that the benefits of equivariance depend on the prediction task and orientation distribution. Across the hemodynamic benchmarks studied, AB-GATr consistently outperforms the evaluated non-equivariant architectures, including variants trained with rotational augmentation.

\subsubsection*{Acknowledgments}
This work made use of the Dutch national e-infrastructure, in particular the Dutch supercomputer Snellius, with the support of the SURF Small Compute Applications grant no. EINF-16639. 
This project has received funding from the European Union’s Horizon Europe research and innovation programme under grant agreement No 101080947 (VASCUL-AID).

\small
\bibliography{iclr2027_conference}
\bibliographystyle{iclr2027_conference}

\appendix
\section{Implementation details}
\label{app:reproducing}

\subsection{Data standarization}

In all the benchmarks we study, vector-valued targets are equivariant with respect to $O(3)$.
AB-UPT utilises per-axis standardisation to mean of $0$ and a standard deviation of $1$ for vectors~\citep{alkin2025abupt}.
However, such a standardisation technique breaks equivariance.

According to Theorem~\ref{thm:standarisation}, the only construction of vector standardisation that respects $O(3)$-equivariance is when the mean is $0$, and the standard deviation is constant across the axes (in other words, the direction of the vectors \textit{does not} change under standardisation).
As such, to retain $O(3)$-equivariance, we standardise vector labels for equivariant models by dividing by the vector magnitude mean (i.e. divide by a single scalar).
For non-equivariant, we follow per-axis standardisation to a mean of $0$ and standard deviation of $1$, which yielded empirically better results than magnitude standardisation only in our experiments.

\newtheorem{theorem}{Theorem}[section]
\begin{theorem}\label{thm:standarisation}
Let $f: \mathbb{R}^c \mapsto \mathbb{R}^3$ be an $O(3)$-equivariant function, i.e., $(\forall \rho \in O(3)) (f(\rho x) = \rho f(x))$, where $x \in \mathbb{R}^c$.
Let's consider the following function $g(x) = \sigma \odot f(x) + \mu$, which represents the standardisation inverse with mean $\mu \in \mathbb{R}^3$ and standard deviation $\sigma \in \mathbb{R}^3$.
A function $g(x)$ is $O(3)$-equivariant only if $\mu = 0$ and $\sigma_1 = \sigma_2 = \sigma_3$
\end{theorem}

\begin{proof}
A function $g(x)$ is $O(3)$-equivariant if following identity holds for every group action $\rho \in O(3)$ and every input $x \in \mathbb{R}^c$:
\begin{gather*}
g(\rho x) = \rho g(x) \\ \\
g(Ax) = \sigma \odot f(\rho x) + \mu = \sigma \odot (\rho f(x)) + \mu \\
\rho g(x) = \rho(\sigma \odot f(x) + \mu) = \rho(\sigma \odot f(x)) + \rho\mu \\ \\
\sigma \odot (\rho f(x)) + \mu = \rho (\sigma \odot f(x)) + \rho \mu
\end{gather*}

First, let's set $f(x) = 0$.
Then, $(\forall A \in O(3))(\mu = \rho \mu)$, which means that value of $\mu$ is invariant under $O(3)$.
The only such vector is the zero vector.
As such for $g(x)$ to be $O(3)$-equivariant, $\mu = 0$.

By substituting $\mu = 0$ we get that $\sigma \odot (\rho f(x)) = \rho(\sigma \odot f(x))$. 
By writing $\sigma$ as a diagonal matrix $D = diag(\sigma_1,\sigma_2,\sigma_3)$, we get that $g(x)$ is $O(3)$-equivariant if $(\forall \rho \in O(3)) (D\rho f(x) = \rho Df(x) \implies D \rho = \rho D)$.
This means that $D$ and $A$ are commuting matrices.
However, according to Schur's lemma, the only commutant of $O(3)$ group is of the form $cI$, where $c$ is any scalar.
As such, $\sigma_1 = \sigma_2 = \sigma_3$.

Consequently, $g(x)$ is $O(3)$-equivariant only if $\mu = 0$ and $\sigma_1 = \sigma_2 = \sigma_3$.
\end{proof}

\subsection{Input features}
\label{app:features}
For ShapeNet-Car, we follow~\cite{alkin2025abupt} and use only Cartesian coordinates as input features for both surface and volumetric representations.
Across all experiments and benchmarks, the Cartesian coordinates are scaled to fit in $[0, 1000]$ cube, following~\cite{alkin2025abupt}.

For the hemodynamic benchmarks, we adopt the input feature design proposed in~\cite{suk2024mesh, rygiel2025wss}, which has been shown to improve model expressiveness. Surface features include Cartesian coordinates, normal vectors, and surface geodesic distances to the inlet and outlets. Volumetric features consist of Cartesian coordinates and volume geodesic distances to the vessel wall, inlet, and outlets. For the AAA-100 dataset, we additionally incorporate principal curvatures, inflow rate, and flow prior features as described in~\cite{rygiel2025wss}.

AB-GATr and LaB-GATr employ projective geometric algebra $\mathbf{G}(3,0,1)$ to embed geometric quantities as $16$-dimensional multivectors, following Equation~\ref{eq:multivector} and Table~\ref{tab:GA_embedding}. In these models, vector-valued features (normal vectors and flow priors) are embedded as \emph{planes}, Cartesian coordinates as \emph{points}, and scalar features (geodesic distances, curvatures, and inflow rate) are provided as auxiliary scalars.

\begin{equation}\label{eq:multivector}
\begin{aligned}
    x \hspace{1pt} = \hspace{1pt} (x_s, \underbrace{x_0, x_1, x_2, x_3}_{\text{vectors}}, \underbrace{x_{01}, x_{02}, x_{03}, x_{12}, x_{13}, x_{23}}_{\text{bivectors}},
    \underbrace{x_{012}, x_{013}, x_{023}, x_{123}}_{\text{trivectors}}, x_{0123})
\end{aligned}
\end{equation}

\begin{table}[!ht]
    \centering
    \caption{Embedding of vectors (planes), points and scalars as $16$-dimensional multivectors in PGA (Eq. \eqref{eq:multivector}), the remaining elements of $x$ remain zero.}
    \label{tab:GA_embedding}
      \begin{tabular}{ll}
        \toprule 
        Geometric object / operation & PGA embedding \\

        \midrule
        Scalar $s \in \mathbb{R}$ & $x_s = s$ \\
        Plane with normal $\nu \in \mathbb{R}^3$, offset $\delta\in\mathbb{R}$ & $(x_0, x_1, x_2, x_3) = (\delta, \nu)$ \\
        Point $\rho \in \mathbb{R}^{3}$ & $(x_{012}, x_{013}, x_{023}, x_{123}) = (\rho, 1)$ \\ 
    
        \bottomrule
  \end{tabular}
\end{table}

\subsection{Training details}
\label{app:training}
To isolate the effect of equivariance, we fix architectural hyperparameters across all benchmarks and minimise differences unrelated to symmetry enforcement. 
For each benchmark (ShapeNet-Car, SynthBifCor, ANEUMO, AAA-100), we fix volume (4096, 5000, 5000, 5000) and surface (3586, 5000, 5000, 5000) anchors and tune dropout ($0\%, 0\%, 0\%, 10\%$) and the pooling radius ($9, 9, 1, 1$).
These settings are shared across all architectures within each benchmark.

\textbf{AB-UPT and AB-GATr} use a single geometry-branch block, followed by $9$ shared blocks alternating between standard and KV-exchange (geometric) anchor attention, and $2$ separate decoder blocks for the surface and volume branches - resulting in $12$ attention blocks per branch with $10$ sharing weights.
Due to differences between standard linear layers and geometric algebra layers, the two models employ different hidden feature dimensions.
AB-UPT uses a hidden dimension of $192$ with $3$ attention heads, following the configuration of~\cite{alkin2025abupt}. 
In contrast, AB-GAT uses $16$ multivector channels, $64$ scalar channels, and $4$ attention heads. 
Total number of trainable parameters is  $7$M and $1.6$M AB-UPT and AB-GATr, respectively.

\textbf{LaB-GATr} utilizes $12$ attention blocks and follows the same transformer block configuration as AB-GATr, with $192$ hidden channels and $4$ attention heads.
Total number of trainable parameters is $1.3$M.

\textbf{Transolver++} utilizes $8$ layers with hidden dimension of $192$, $4$ attention heads and number of slices set to $64$.
Total number of trainable parameters is $2$M.

All models are trained with a batch size of $1$, gradient clipping with a threshold of $1.0$, and the LION optimiser~(\cite{lion}) with weight decay $5\times10^{-2}$. 
We use a linear learning-rate warmup for $5\%$ of the training schedule, followed by cosine decay to a final learning rate of $10^{-6}$. 
The peak learning rate is tuned separately for each model and benchmark in the range $[10^{-5}, 10^{-3}]$ (see Table~\ref{resources-table}). 
All models are trained in full \texttt{float32} precision, as we observed instability when using mixed \texttt{float16} precision for AB-GATr. Mean squared error is used as the training loss for both surface and volume quantities, with the final loss being a mean of both branch losses.

\begin{table}[!h]
  \caption[Short Heading]{Training details and resources. For non-joint models, we denote volume variant by "V" and surface variant by "S". Column "Time [h]" refers to total training time and "VRAM [GB]" to peak GPU memory usage.}
  \label{resources-table}
  \centering
  \begin{tabular}{llcccc}
    \toprule
    Benchmark & Model & LR & Epochs & Time [h] & VRAM [GB] \\
    \midrule

    \multirow{4}{*}{SynthBifCor}
        & AB-UPT & 2e-5 & 1000 & 34  & 2.3 \\
        & Transolver++ (V | S) & 2e-5 & 1000  & 104 | 33 & 24.3 | 2.7 \\
        & LaB-GATr (V | S) & 5e-5 & 300  & 285 | 39 & 10.6 | 3.9 \\
        & AB-GATr & 1e-4 & 300 & 70 & 9.5 \\
    \midrule

    \multirow{4}{*}{AAA-100}
        & AB-UPT & 5e-5 & 500 & 1 & 1.9 \\
        & Transolver++ (V | S) & 5e-5 & 500 & 8.3 | 2.4 & 71 | 4.5 \\
        & LaB-GATr (V | S) & 2e-4 & 1000 & 110 | 14 & 20 | 4.9 \\
        & AB-GATr & 2e-4 & 1000 & 11 & 7 \\
    \midrule

    \multirow{4}{*}{ANEUMO} 
        & AB-UPT & 5e-5 & 100 & 17 & 1.9 \\
        & Transolver++ (V) & 5e-5 & 100 & 59 & 43 \\
        & LaB-GATr (V) & 5e-5 & 50 & 226 & 9.7 \\
        & AB-GATr & 1e-4 & 50 & 51 & 6.5 \\

    \bottomrule

\end{tabular}
\end{table}

\subsection{Scalability \& resources}
\label{app:resources}
All the models were trained on a single NVIDIA H100/80GB, taking approximately the following resources as denoted in Table~\ref{resources-table}.

As reported in Table~\ref{resources-table}, AB-UPT offers substantially better time and memory usage than Transolver++.
This fact makes it far more optimal to serve as a backbone of an equivariant network due to the usual computational overhead that comes with the introduction of equivariant layers.
By adopting components of AB-UPT, AB-GATr achieves similar scaling capabilities (with equivariance overhead introduced by GATr).
AB-GATr allows for the training of geometries of arbitrary size, due to its training memory complexity being dependent on the number of anchor tokens only.
During inference, the memory complexity depends on the number of query tokens as well; however, a chunked inference scheme (as described in~\cite{alkin2025abupt}) can be employed to split all query tokens into smaller subsets that fit in GPU memory.
This is a substantial gain over both Transolver++ and LaB-GATr, whose training memory complexity depends on the number of all tokens.
Moreover, LaB-GATr's initial subsampling employs a costly farthest-point sampling scheme that scales poorly with the number of tokens, which is reflected by the high training times as reported in Table~\ref{resources-table}.

\subsection{Matching training compute across architectures}
Table~\ref{resources-table} shows that training costs differ across architectures on the same benchmark. 
We explored compute-matched configurations by increasing the training budgets of non-equivariant models to those of AB-GATr. 
In these experiments, performance plateaued or deteriorated as overfitting increased, particularly on AAA-100. 
We therefore report the selected best-performing training configurations alongside their runtimes and memory requirements, allowing predictive accuracy to be assessed together with computational cost. 
The reported comparisons consequently use different training budgets.

\section{CFD benchmarks}

\subsection{Benchmark details}
\label{app:datasets}
In Table~\ref{dataset-table}, we provide dataset details including their complexity, size and available target fields.
In our experiments, we have left out pressure fields available for hemodynamic datasets to focus on vector-valued fields.
However, the experiments and models can be naturally extended to predict pressure, as used in ShapeNet-Car.

In Fig~\ref{fig:datasets-projections} we expand upon Fig~\ref{fig:data} to include point occupancy histogram in all Cartesian axes.

\begin{table}[!ht]
  \caption[Short Heading]{Benchmark details. By | we separate volume (left) and surface (right) quantities. The dataset split represents \texttt{train:valid:test} dataset sizes. We list all available prediction fields per dataset and mark ones used in our benchmarks in \textbf{bold}.}
  \label{dataset-table}
  \centering
  \begin{tabular}{llllll}
  \hline
        \toprule 
        \multicolumn{1}{c}{Name} & \multicolumn{1}{c}{Type} & \multicolumn{1}{c}{\#points (mean $\pm$ std)} &
        \multicolumn{1}{c}{Dataset split} & \multicolumn{1}{c}{Fields}  & \multicolumn{1}{c}{Timepoints}\\

        \midrule
        
        \multicolumn{1}{l}{SynthBifCor} & \multicolumn{1}{c}{hemo} & \multicolumn{1}{c}{150k $\pm$ 34k | 17k $\pm$ 2.3k} & \multicolumn{1}{c}{1600:200:200$^2$} & \multicolumn{1}{c}{$\boldsymbol{u}, p$ | $\boldsymbol{\tau}$} & \multicolumn{1}{c}{1}\\

        \multicolumn{1}{l}{AAA-100} & \multicolumn{1}{c}{hemo} & \multicolumn{1}{c}{340k $\pm$ 115k | 30k $\pm$ 5.5k} & \multicolumn{1}{c}{90:0:10$^1$} & \multicolumn{1}{c}{$\boldsymbol{u}, p$ | $\boldsymbol{\tau}$} & \multicolumn{1}{c}{21}\\

        \multicolumn{1}{l}{ANEUMO} & \multicolumn{1}{c}{hemo} & \multicolumn{1}{c}{186k $\pm$ 53k | 10k $\pm$ 2.7k} & \multicolumn{1}{c}{8495:1072:1093$^3$} & \multicolumn{1}{c}{$\boldsymbol{u}, p$ | -} & \multicolumn{1}{c}{1}\\

        \multicolumn{1}{l}{ShapeNet-Car} & \multicolumn{1}{c}{aero} & \multicolumn{1}{c}{29k | 3.6k } & \multicolumn{1}{c}{789:0:100$^4$} & \multicolumn{1}{c}{$\boldsymbol{u}$ | $p$} & \multicolumn{1}{c}{1} \\
        
        \bottomrule
        \multicolumn{6}{l}{$^1$10-fold cross-validation} \\
        \multicolumn{6}{l}{$^2$from \cite{suk2023velocity}} \\
        \multicolumn{6}{l}{$^3$stratified according to the underlying base anatomies} \\
        \multicolumn{6}{l}{$^4$from \cite{alkin2024upt}}
  \end{tabular}
\end{table}

\begin{figure*}[ht!]
    \centering
    \includegraphics[width=\textwidth]{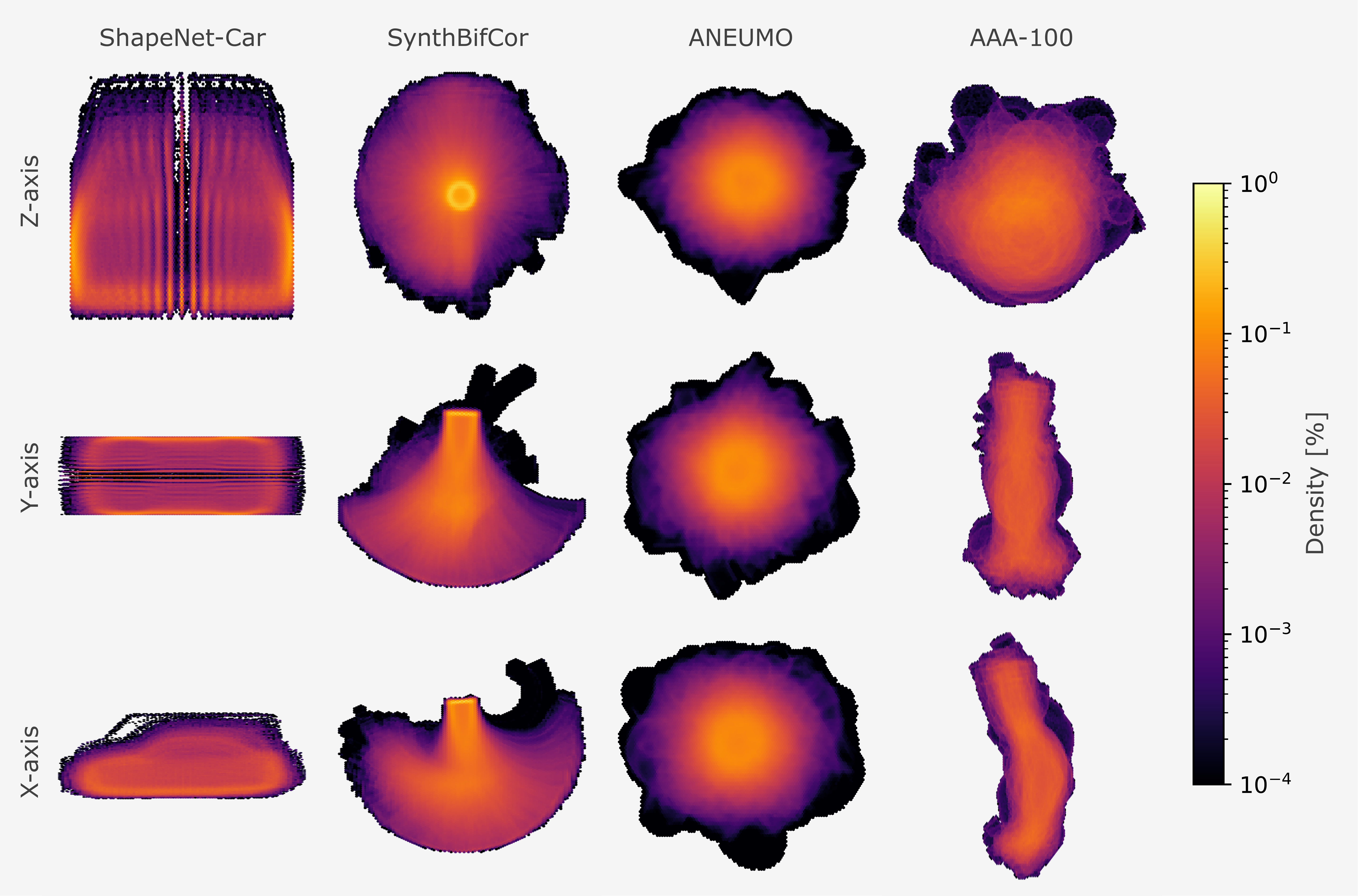}
    \caption{Canonical alignment across different benchmarks: a histogram of
    point occupancy across all $3$ Cartesian axes.}
    \label{fig:datasets-projections}
\end{figure*}

\subsection{On distributional symmetry breaking}
To compute $m(p_X \circ r_\Theta)$, we trained thin classifiers (cross-attention; ShapeNet-Car, SynthBifCor and AAA-100: 108,354 parameters, points and surface normal features; Aneumo: 107,970 parameters, points features) until perceived convergence on the validation splits (ShapeNet-Car: 92,520 iterations, SynthBifCor: 44,800 iterations, Aneumo: 135,920 iterations) or training split (AAA-100: 2,967 iterations) with batch size four and Adam optimiser.
These binary classifiers are trained to distinguish original samples from samples transformed under $\mathrm{SO}(3)$.

In Table~\ref{tab:distributional-symmetry-breaking} we report the integrated symmetry response $\int_0^\pi m(p_X \circ r_\Theta) \mathrm{d}\theta$ with $\Theta \sim \mathcal{U}_{[0, \theta]}$ and $r_\Theta \in \{R_{\theta'} \in \mathrm{SO(3)} \mid \theta' \le \Theta\}$ where rotations $R_{\theta'}$ are described by their angle $\theta'$.
In other words, $r_\Theta$ is a random rotation up until a maximum angle $\theta$.
This provides a measure of how close the dataset is to being perfectly aligned, or on the contrary, fully symmetric.
Moreover, this measure can be applied a priori, without having to train a task-specific model.
We find that ShapeNet-Car and SynthBifCor are more strongly aligned relative to AAA-100, indicated by higher values of $\int m(p_X \circ r_\Theta)$.

We also report task-useful distributional symmetry breaking as introduced by \citet{lawrence2026to}.
We compute this from mean absolute errors on the test splits (ShapeNet-Car: $p$, SynthBifCor: $\tau$) or 10-fold cross-validation (AAA-100: $\tau$) after training AB-UPT (ca. 400,000 parameters, points and surface normal features) until perceived convergence on the validation splits (ShapeNet-Car: 180,000 iterations, SynthBifCor: 56,800 iterations) or training split (AAA-100: 10,534 iterations) with batch size four and Adam optimiser.
Task-useful distributional symmetry breaking requires passing inputs through a canonicalisation function for which we chose a randomly initialised and frozen GATr geometric product layer where we break reflection equivariance by keeping a fixed reference multivector and translation equivariance by encoding the point-wise distance to the origin (after centering).
This leaves only $\mathrm{SO(3)}$ equivariance intact; breaking the other symmetries was necessary for our experiments.
The canonicalisation function outputs features in equivariant as well as invariant pathways.
We find that ShapeNet-Car labels are strongly coupled with its canonical alignment followed by SynthBifCor and least of all, AAA-100.

We note that both presented metrics correlate with the efficacy of AB-GATr over AB-UPT (compare Table~\ref{hemodynamics-table}), i.e., equivariance.
In the strongly canonically aligned dataset ShapeNet-Car, labels can be inferred from this distributional symmetry breaking which provides a powerful shortcut when test samples exhibit the same alignment as training samples.
This mechanism becomes less and less powerful as alignment gets weaker or labels no longer correlate, e.g., for the ANEUMO.

\label{app:distirbutional-symmetry}
\begin{table}
    \centering
    \caption{Symmetry response $\int m(p_X) \in [0, \pi]$ and task-useful distributional symmetry breaking $t(p_{X, Y})$ measure canonical alignment of the dataset ($\uparrow$) and its influence on the task ($\uparrow$), respectively. We show the mean over the test split (ShapeNet-Car, SynthBifcor) or 10-fold cross-validation (AAA-100).}
    \label{tab:distributional-symmetry-breaking}
    \begin{tabular}{lccc}
        \toprule 
        & ShapeNet-Car & SynthBifCor & AAA-100 \\
        \midrule
        $\int m(p_X \circ r_\Theta) - \frac{\pi}{2}$ & 0.939 & 0.863 & 0.666 \\
        $t(p_{X, Y})$ & 29.23 & 8.52 & 2.35 \\
        \bottomrule
    \end{tabular}
\end{table}

\subsection{Additional qualitative results}
\label{app:extended-visuals}
In this appendix we include additional qualitative results for SynthBifCor~(Fig.~\ref{fig:extended-synthbifcor}), AAA-100~(Fig.~\ref{fig:extended-aaa}) and ANEUMO~(Fig~\ref{fig:extended-aneumo}).
Each figure includes three examples that belong to roughly the first quartile (Q1), second quartile/median (Q2) and third quartile (Q3) of the relative L2 error distribution.
They can be treated as easy, median and hard samples, respectively.
The presented examples have been predicted by the models in their original dataset alignment.
Each example is accompanied by the velocity contour plot and absolute error plot with respect to GT (CFD).

Additionally, in Fig.~\ref{fig:aneumo-rotated} we showcase changes in prediction quality over the anchored-branched model family with different sample orientations. 
Non-equivariant models yield different error rates based on the sample orientation, while equivariant AB-GATr offers the same level of predictive quality.
We can also observe that AB-UPT trained with no data augmentations yields substantially different error maps across different orientations.

\label{app:visuals}
\begin{figure}[h!]
    \centering
    \includegraphics[width=\textwidth]{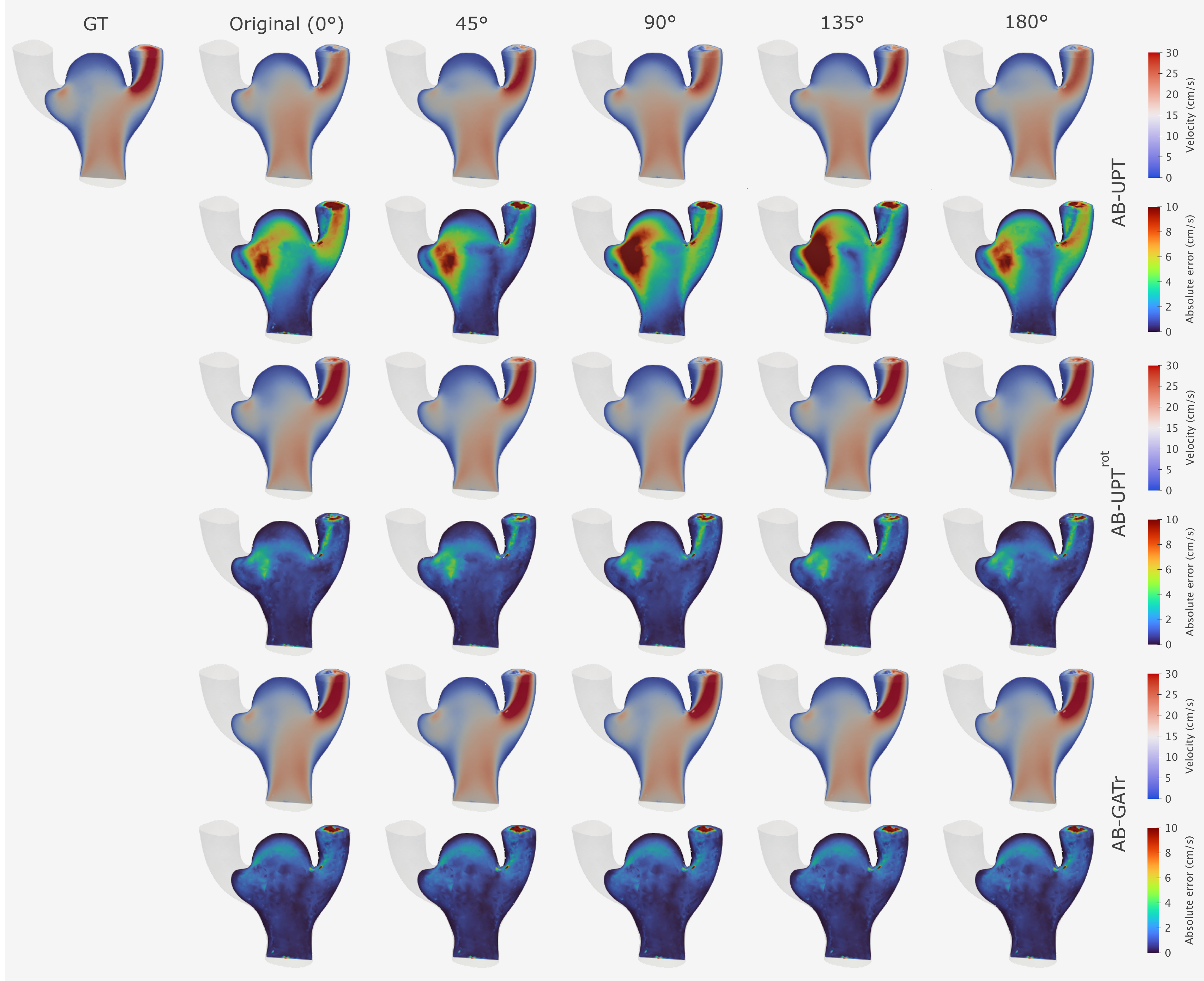}
    \caption{Qualitative comparison of neural surrogate velocity predictions single ANEUMO sample with different rotations applied ($0^\circ$, $45^\circ$, $90^\circ$, $135^\circ$, $180^\circ$). For the sake of visualization all samples are represented in the original alignment.}
    \label{fig:aneumo-rotated}
\end{figure}
\begin{figure}[h!]
    \centering
    \includegraphics[width=\textwidth]{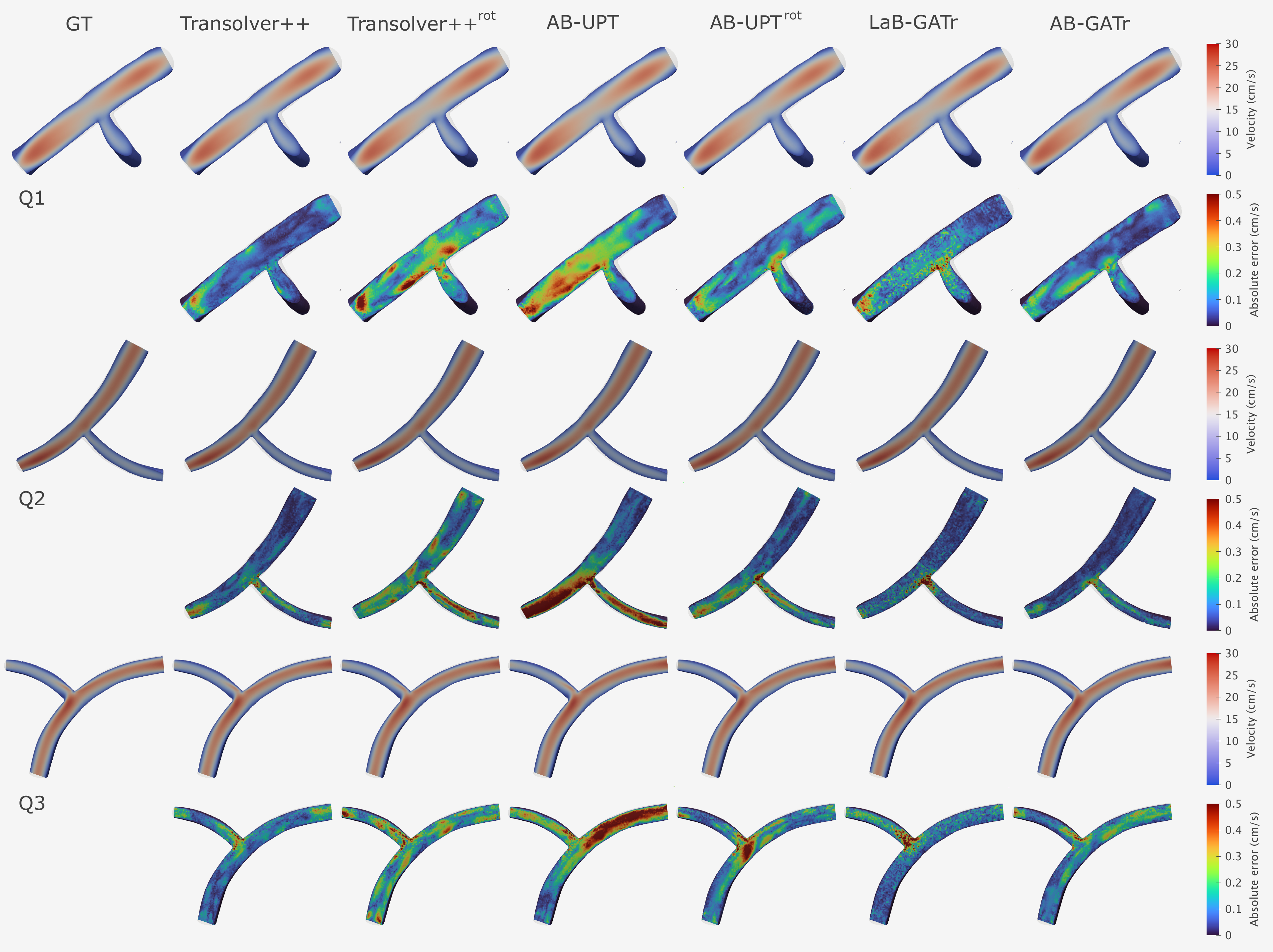}
    \caption{Extended qualitative comparison of neural surrogate velocity predictions on originally aligned SynthBifCor. Samples Q1, Q2 and Q3 correspond roughly to first, second (median) and third quartile of relative L2 error distribution.}
    \label{fig:extended-synthbifcor}
\end{figure}
\begin{figure}[h!]
    \centering
    \includegraphics[width=\textwidth]{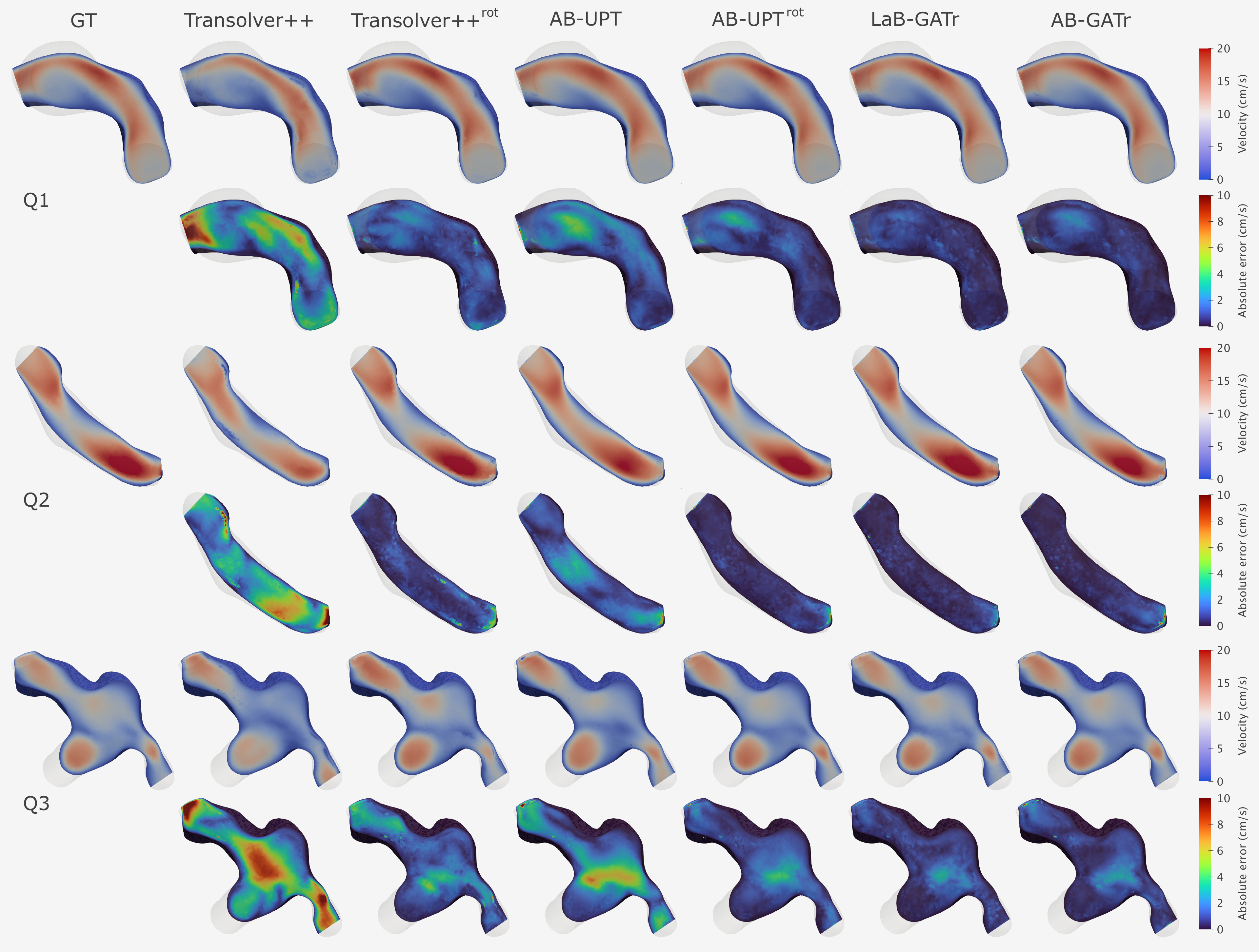}
    \caption{Extended qualitative comparison of neural surrogate velocity predictions on originally aligned ANEUMO. Samples Q1, Q2 and Q3 correspond roughly to first, second (median) and third quartile of relative L2 error distribution.}
    \label{fig:extended-aneumo}
\end{figure}
\begin{figure}[h!]
    \centering
    \includegraphics[width=\textwidth]{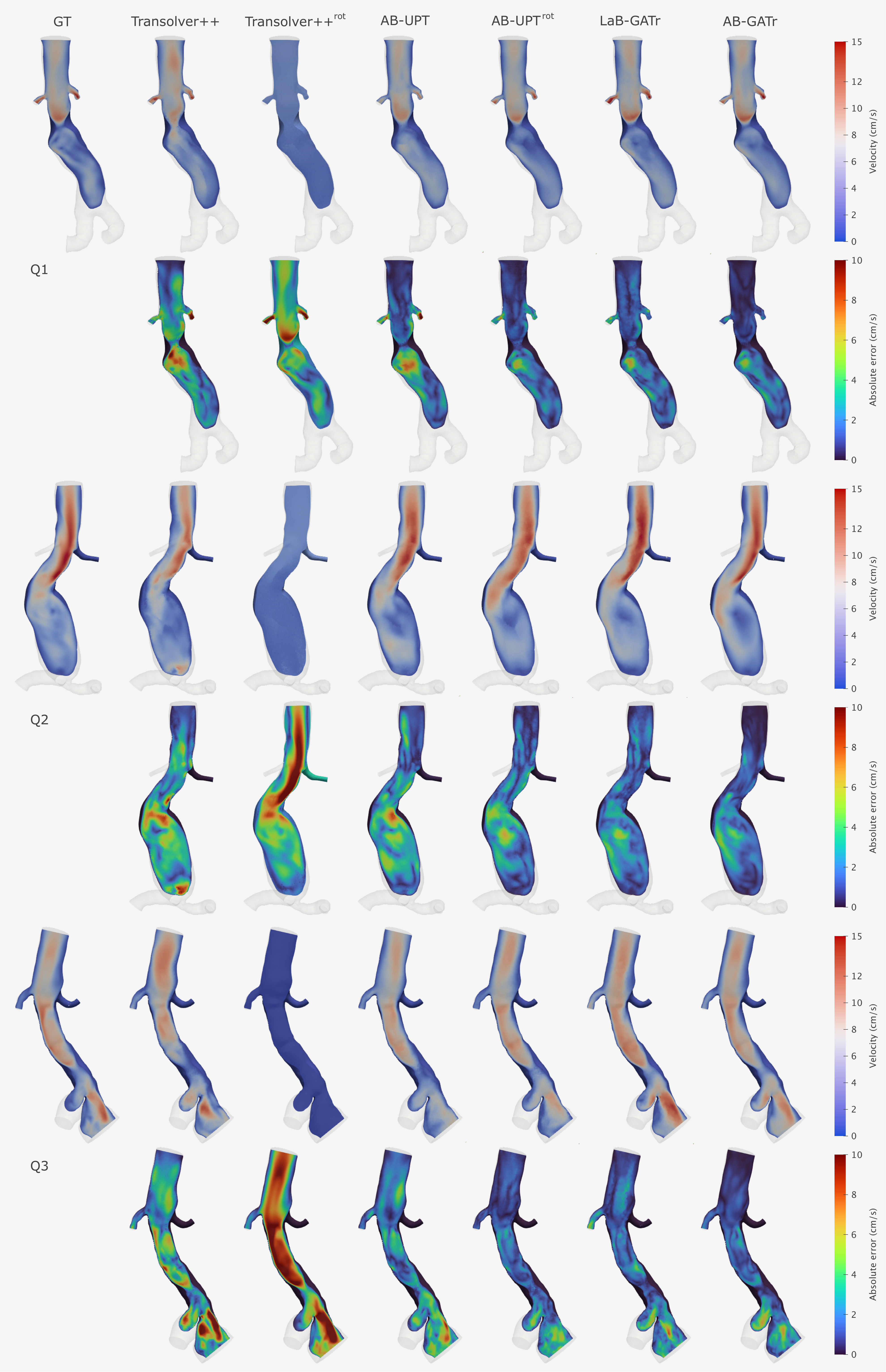}
    \caption{Extended qualitative comparison of neural surrogate velocity predictions on originally aligned AAA-100. Samples Q1, Q2 and Q3 correspond roughly to first, second (median) and third quartile of relative L2 error distribution.}
    \label{fig:extended-aaa}
\end{figure}



\end{document}